\documentclass{article}
\usepackage[utf8]{inputenc}
\usepackage{amsmath, amssymb, amsthm}
\usepackage{algorithm}
\usepackage[noend]{algpseudocode}
\usepackage{graphicx}
\usepackage[margin=1in]{geometry}
\usepackage{caption}
\usepackage{xcolor}
\usepackage{blindtext}
\usepackage{authblk}
\usepackage{mathtools}
\usepackage[sortcites,sorting=none,style=numeric,backend=biber]{biblatex}
\usepackage[hidelinks]{hyperref}
\usepackage[capitalize,noabbrev]{cleveref}

\newtheorem{theorem}{Theorem}
\newtheorem{lemma}{Lemma}
\newtheorem{definition}{Definition}
\newtheorem{proposition}{Proposition}

\newtheorem{assumption}{Assumption}

\title{Efficient Approximation of Volterra Series for High-Dimensional Systems}
\author[1,2]{Navin Khoshnan\thanks{navin.khoshnan@cgu.edu}}
\author[1,2]{Claudia K Petritsch}
\author[1,2]{Bryce Allen Bagley}
\affil[1]{Mathematical Medicine Group, Stanford University}
\affil[2]{Department of Neurosurgery, Stanford Medical School}

\begin{document}

\maketitle

\begin{abstract}
    The identification of high-dimensional nonlinear dynamical systems via the Volterra series has significant potential, but has been severely hindered by the curse of dimensionality. Tensor Network (TN) methods such as the Modified Alternating Linear Scheme (MVMALS) have been a breakthrough for the field, offering a tractable approach by exploiting the low-rank structure in Volterra kernels. However, these techniques still encounter prohibitive computational and memory bottlenecks due to high-order polynomial scaling with respect to input dimension. To overcome this barrier, we introduce the Tensor Head Averaging (THA) algorithm, which significantly reduces complexity by constructing an ensemble of localized MVMALS models trained on small subsets of the input space. In this paper, we present a theoretical foundation for the THA algorithm. We establish  observable, finite-sample bounds on the error between the THA ensemble and a full MVMALS model, and  we derive an exact decomposition of the squared error. This decomposition is used to analyze the manner in which subset models implicitly compensate for omitted dynamics. We quantify this effect, and prove that correlation between the included and omitted dynamics creates an optimization incentive which drives THA's performance toward accuracy superior to a simple truncation of a full MVMALS model. THA thus offers a scalable and theoretically grounded approach for identifying previously intractable high-dimensional systems. 
\end{abstract}

\section{Introduction}
\label{sec:intro}

The identification and modeling of high-dimensional nonlinear dynamical systems represents a significant challenge across numerous domains in science and engineering. From understanding physiological systems to advanced control engineering and signal processing,\cite{marmarelis_analysis_1978, rugh_nonlinear_1981, korenberg_identification_1996, marmarelis_nonlinear_2004, zhu_behavioral_2016, doyle_identification_2002, favier_nonlinear_2012, wambacq_high-frequency_1999,abbas_adaptive_2005,chatterjee_non-linear_2004,korenberg_orthogonal_1991,palm_volterra_1985,} the ability to characterize and classify input-output relationships without relying on linearization is helpful for capturing various phenomena where memory and nonlinearity are central. For discrete-time, causal, time-invariant systems with fading memory, the Volterra Series (VS) offers a natural extension of linear system theory.\cite{schetzen_volterra_1980, boyd_fading_1985, palm_volterra_1985,bedrosian_output_1971,} Serving as a functional analogue to the Taylor expansion for static functions, the VS generalizes the convolution integral, representing the system output as a sum of multidimensional convolutions. The characterizing functions of these convolutions,  known as Volterra kernels, generalize the linear impulse response of a system to capture higher-order interactions and memory effects. Based on generalizations of the Stone-Weierstrass theorem to functionals, the VS provably acts as a universal approximator.\cite{boyd_fading_1985,schetzen_volterra_1980,palm_volterra_1985} This structure provides a nonparametric representation capable of approximating a variety of nonlinear behaviors, while retaining greater interpretability compared to black-box methods.

However, the practical application of the VS has been severely limited by the curse of dimensionality. For a Multiple-Input Multiple-Output (MIMO) system with \(p\) inputs, memory \(M\), and system degree \(d\),  The number of parameters required to characterize the kernels scales exponentially as \(O((pM)^d)\). This explosion in dimensionality creates challenges across the identification process. Computationally, the time complexity of traditional identification algorithms that operate on the full parameter space, such as direct least squares estimation, scales exponentially with degree \(d\). Memory requirements for storing such large kernels quickly become prohibitive, and the resultant exponential complexity renders traditional VS identification methods intractable for all but the smallest systems.

The recognition that Volterra kernels are inherently high-order tensors has led to a shift in modeling approaches through the introduction of Tensor Network (TN) methods.\cite{kolda_tensor_2009, cichocki_tensor_2017} TN techniques, such as the Tensor Train (TT) format introduced by Oseledets,\cite{oseledets_tensor-train_2011} exploit the low-rank structure often present within these kernels to achieve significant compression. This approach allows for a parametrization that scales linearly rather than exponentially, with the degree \(d\), effectively mitigating the curse of dimensionality associated with the system's size and complexity. The Modified Alternating Linear Scheme (MVMALS) algorithm, adapted for Volterra series identification by Batselier et al.,\cite{batselier_tensor_2017} provides an efficient and numerically stable method for identifying VS models of systems within this compressed TN framework. Despite these important advances, a critical bottleneck still remains for systems with high input dimension. While the exponential dependence on \(d\) is resolved, the time and memory complexity of MVMALS scales with the input space polynomially as \(p^6\). In areas characterized by large input dimensionality such as biological systems or large-scale sensor arrays, this high-order polynomial complexity remains prohibitive. 

To address this scalability gap, we introduce the Tensor Head Averaging (THA) algorithm. THA utilizes a structure analogous to the MapReduce paradigm\cite{map_reduce} in order to drastically reduce complexity for systems with large input dimensionality. Similar to the ``map" phase, THA constructs an ensemble of smaller, localized TN models known as heads, where each head is trained via MVMALS on a distinct small subset of the input variables. The final prediction is obtained by averaging the outputs of these individual heads, corresponding to the ``reduce" phase. This approach replaces  the dependence on the full input dimension \(p\) with one on the much smaller subset size \(k \ll p\), enabling the analysis of systems previously considered to be intractable.

A phenomenon inherent to this localized approach is a form of omitted variable bias, which we term ``baking." When a subset model is optimized during training, it does not merely capture the truncated dynamics corresponding to the included variables. Instead, the optimization process inherently adjusts the parameters of the subset model to compensate for the dynamics of the omitted variables, provided there is correlation in the data. This process effectively ``bakes" the influence of the global system dynamics into the localized model parameters. Understanding this mechanism is necessary to properly characterize the approximation capabilities of the THA ensemble.

This paper establishes the theoretical results for the THA algorithm, focusing on the approximation error of the THA ensemble relative to a full MVMALS model, and the  optimization dynamics that induce baking. The primary contributions of this paper are summarized as follows: 

\begin{enumerate}
    \item \textbf{Observable Error Bound:} We establish a numerically computable bound on the error between the THA prediction and that of a full MVMALS model. The bound is in a finite-sample setting without assumptions of convexity or access to global optimizers. This bound is not an a priori error prediction, but rather allows one to understand the sources of error in a THA approximation.
    \item \textbf{Geometric Error Decomposition and Baking Gain:} Using the geometry induced by the empirical Frobenius norm, we derive an exact decomposition of the squared error. 
    \item \textbf{Optimization Behavior and the Incentive for Baking:} We analyze the MVMALS optimization process within the THA context. We prove that under standard regularity conditions aligned with MALS theory, the correlation between included and omitted dynamics creates a gradient incentive that drives the optimization away from a simple truncation, provided this correlation is not orthogonal to the tangent space of the TN manifold. We further show that the condition preventing this incentive is a measure-zero event.
    \item \textbf{Computational Complexity:} We analyze the computational complexity of THA, demonstrating a significant reduction in computational cost compared to the full MVMALS approach.
\end{enumerate}

Section~\ref{sec:preliminaries} establishes the necessary notation and background on Volterra series and tensor networks. Section~\ref{sec:tha_framework} formally defines the THA algorithm and analyzes its computational complexity. Section~\ref{sec:error_bounds} presents the main results on the observable error bounds and the geometric error decomposition. Section~\ref{sec:optimization_main} provides a look into the optimization dynamics and the incentive for baking.

\section{Preliminaries}
\label{sec:preliminaries}

Here we review the notation and basic concepts that are foundational to MIMO Volterra series, their tensor representation, and the empirical measures used for later analysis.

\subsection{MIMO Volterra Series and Empirical Norms}

We consider a discrete-time dynamical system with \(p\) inputs and \(l\) outputs, observed over \(t=1,\dots,N\). The system is modeled by a finite Volterra series (VS) of degree \(d\) and memory \(M\).

%\begin{definition}[Lagged Input Vector]
The input vector \(\mathbf{x}_t \in \mathbb{R}^{q}\), where \(q=pM+1\), aggregates past inputs \(\mathbf{u}(t)\in\mathbb{R}^p\), augmented by a constant:

    \[
    \mathbf{x}_t \coloneqq \big[1, \mathbf{u}(t)^\top, \mathbf{u}(t-1)^\top, \dots,\mathbf{u}(t-M+1)^\top \big]^\top
    \]
    
%\end{definition}

%\begin{definition}[Volterra Feature Vector]
The degree-\(d\) Volterra feature vector \(\boldsymbol{\phi}_t \in \mathbb{R}^{Q}\), where \(Q=q^d\), is the \(d\)-th order Kronecker power: \(\boldsymbol{\phi}_t \coloneqq \mathbf{x}_t^{\otimes d}\).
%\end{definition}

The VS is linear in its parameters, organized into a coefficient matrix \(B \in \mathbb{R}^{Q\times l}\). The system response is \(\mathbf{y}(t)^\top = \boldsymbol{\phi}_t^\top B\). We define the design matrix \(U\in\mathbb{R}^{N\times Q}\) as \(U = [\boldsymbol{\phi}_1, \dots, \boldsymbol{\phi}_N]^\top\) and the output matrix \(Y = [\mathbf{y}(1), \dots, \mathbf{y}(N)]^\top \in \mathbb{R}^{N\times l}\). The identification problem is estimating \(B\) from the linear system \(Y = U B\). The exponential scaling of \(Q\) motivates the use of compressed representations.

To analyze the finite-sample approximation error, we define the following:

\begin{definition}[Empirical Frobenius Norm and Inner Product]
The empirical Frobenius norm for \(A \in \mathbb{R}^{N\times l}\) is \(\|A\|_{F,N} \coloneqq \frac{1}{\sqrt{N}}\|A\|_F\). The associated inner product between \(A, C \in \mathbb{R}^{N\times l}\) is \(\langle A, C \rangle_{F,N} \coloneqq \frac{1}{N}\mathrm{Tr}(A^\top C)\).
\end{definition}

\begin{definition}[Empirical Gram Matrix]
The empirical Gram matrix \(G\in\mathbb{R}^{Q\times Q}\) of the features is \(G \ \coloneqq \frac{1}{N}U^\top U\).
\end{definition}
\noindent The matrix \(G\) is positive semidefinite (\(G \succeq 0\)). Its spectral properties characterize the richness of the input data and are important for the error bounds established later.

\subsection{Tensor Networks and MVMALS}

To address the high dimensionality of the coefficient matrix \(B \in \mathbb{R}^{Q\times l}\), we employ a compressed representation based on the TN formalism, related to the TT format (Oseledets, 2011). The matrix \(B\) is viewed as the matricization of a \((d+1)\)-order tensor \(\mathcal{B} \in \mathbb{R}^{l \times q \times \cdots \times q}\), which is represented as a contracted product of lower-order tensors (cores).

%\begin{definition}[Tensor Network Parameterization]
We adopt a TN parametrization where \(\mathbf{r}=(r_1, \dots, r_{d-1})\) are the TN-ranks, with boundary ranks \(r_0=l\) (outputs) and \(r_d=1\). The parametrization utilizes \(d\) cores, \(\mathcal{U}=(V^{(1)}, \dots, V^{(d)})\), where the \(k\)-th core is \(V^{(k)} \in \mathbb{R}^{r_{k-1}\times q\times r_k}\). A multilinear map \(\tau\) constructs the full tensor \(\mathcal{B} = \tau(\mathcal{U})\) via contraction (see Appendix~\ref{app:tn_details} for the explicit formula).
%\end{definition}

If \(\rho = \max_k r_k\) is the maximal TN-rank, the TN format requires \(O(d q \rho^2 + lq\rho)\) parameters. This scales linearly with the degree \(d\), offering significant compression compared to the \(O(l q^d)\) parameters of the dense representation.

The set of tensors with ranks bounded by \(\mathbf{r}\) is denoted \(\mathcal{T}_{\le \mathbf{r}}\). The subset with ranks exactly \(\mathbf{r}\), denoted \(\mathcal{T}_{\mathbf{r}}\). It is established that this set forms a smooth, embedded submanifold under standard regularity conditions,\cite{holtz_manifolds_2012, uschmajew_geometry_2013} and this geometric structure is crucial for the analysis of optimization on the TN manifold.

The identification problem is formulated as the minimization of the empirical loss subject to the TN structure:
\begin{equation}
\min_{B \in \mathcal{T}_{\le \mathbf{r}}} L(B) \coloneqq \|Y - UB\|_{F,N}^2.
\label{eq:optimization_problem}
\end{equation}

This is a non-convex optimization problem due to the multilinearity of \(\tau\).

We use the MVMALS algorithm introduced by Batselier et al.\cite{batselier_tensor_2017} to find a local solution to Eq.~\eqref{eq:optimization_problem}. MVMALS is an adaptation of the Alternating Least Squares (ALS) principle that iteratively optimizes the cores by solving reduced linear least-squares problems. Holtz et al., have shown it ensures numerical stability through orthogonalization constraints\cite{holtz_alternating_2012} and Rohwedder \& Uschmajew have shown it guarantees monotonic convergence to a block-stationary point\cite{rohwedder_local_2013}.

\begin{definition}[Full MVMALS Model Parameters \(B^\star\)]
We define the full MVMALS model as the model identified by applying the MVMALS algorithm to the dataset \((U, Y)\) utilizing the complete set of \(p\) input variables. The resulting parameters constitute the coefficient matrix \(B^\star \in \mathcal{T}_{\mathbf{r}} \subset \mathbb{R}^{Q\times l}\). The predictions of this full model are denoted as \(\widehat{Y}^{\mathrm{full}} = U B^\star\).
\end{definition}

\noindent Due to the non-convexity of the optimization landscape, \(B^\star\) represents a local stationary point of the loss function \(L(B)\) on the manifold \(\mathcal{T}_{\mathbf{r}}\).

\section{The THA Algorithm}
\label{sec:tha_framework}

We now establish the notation and definitions for the THA algorithm, which approximates a high-dimensional Volterra system by an ensemble of lower-dimensional models, each identified using the Modified Alternating Linear Scheme within a tensor network format.

\subsection{THA Construction and Prediction}
% (Ensemble approach, subsets S_k. The THA prediction Y^tha.)

THA  mitigates the computational complexity associated with the dimension \(Q\) by constructing an ensemble of \(K\) smaller models, referred to as ``heads."

Let \(\mathcal{I} = \{1, \dots, p\}\) be the index set of the input variables. We define a collection of \(K\) subsets of these indices, \(\{S_k\}_{k=1}^K\), where \(S_k \subset \mathcal{I}\).

For each subset \(S_k\), a corresponding model is trained using only the input variables indexed by \(S_k\). This identification is performed using MVMALS\cite{batselier_tensor_2017}, resulting in a parameter matrix \(\widehat{B}^{(k)} \in \mathbb{R}^{Q\times l}\). These parameters are embedded in the full space, such that \(\widehat{B}^{(k)}\) is zero for any monomial in \(\phi_t\) involving variables not indexed by \(S_k\).

The THA ensemble prediction is formed by a weighted average of the individual head predictions. Let \(\omega_1,\dots,\omega_K\ge 0\) be the ensemble weights, satisfying the convexity constraint \(\sum_{k=1}^K\omega_k=1\). The aggregate THA prediction \(\widehat{Y}^{\mathrm{THA}} \in \mathbb{R}^{N\times l}\) is defined as:

    \[
        \widehat{Y}^{\mathrm{THA}} \ \coloneqq \sum_{k=1}^K \omega_k\,(U\,\widehat{B}^{(k)}).
    \]

\subsection{Masking and Coverage}

To analyze the relationship between the subset models and the full parameter space, we introduce operators that formalize the restriction of the feature space to specific subsets
.
\begin{definition}[Mask Operator \(P_{S_k}\)]
For a subset of input indices \(S_k\), the Mask Operator \(P_{S_k} \in \mathbb{R}^{Q\times Q}\) is a diagonal projection matrix. The diagonal entries are defined as:
    
    \[
        (P_{S_k})_{jj} = \begin{cases} 1 & \text{if the } j\text{-th monomial in } \phi_t \text{ only involves variables indexed by } S_k \text{ (including the constant 1)}, \\ 0 & \text{otherwise}. \end{cases}
    \]
\end{definition}

By construction, the parameters of the \(k\)-th head satisfy the constraint \(P_{S_k}\widehat{B}^{(k)} = \widehat{B}^{(k)}\).

The collective effect of the ensemble weighting on the representation of the full feature space is captured by the Coverage Operator.

\begin{definition}[Coverage Operator \(C_\omega\)]

Given the ensemble weights \(\{\omega_k\}\) and the mask operators \(\{P_{S_k}\}\), the Coverage Operator \(C_\omega \in \mathbb{R}^{Q\times Q}\) is the weighted average of the masks:
    \[
        C_\omega  \coloneqq \sum_{k=1}^K \omega_kP_{S_k}.
    \]

\end{definition}

The Coverage Operator \(C_\omega\) is diagonal and positive semidefinite. Since each \(P_{S_k}\) is a projection matrix and the weights are convex, \(C_\omega\) satisfies the operator inequality \(0\preceq C_\omega \preceq I\). This operator quantifies the extent to which each monomial is represented within the ensemble.

\subsection{Pseudocode For the THA Algorithm}

We now outline the steps for the construction and evaluation of the THA ensemble presented in Algorithm~\ref{alg:tha}. For a system with a large input space, the space of possible subsets is vast, making exhaustive enumeration infeasible. A selection strategy is employed to choose \(K\) subsets. THA allows flexibility in the selection strategy, with some examples being random sampling, weighted sampling, or criteria specific to domain of application. 

In the training phase, \(K\) localized models are identified using MVMALS on their respective input subsets \(S_k\). The independence of these tasks allows for them to potentially be run in parallel, and the simultaneous training of all \(K\) models offers a significant reduction in the wall-clock time required for ensemble construction. The minimization of error between the localized model prediction and full system output is what pushes for baking to occur. 

The strategy for optimizing subset weights is also flexible, generally done on a separate validation dataset \(Y_{val}\) to prevent overfitting. The trained heads are then simulated on the validation data, and the weights \(\omega_k\) are determined. The outline showcases this through the implementation of constrained least squares as an example strategy.

Finally, the outputs of the trained heads are aggregated using the optimized weights for inference on new data. 

\begin{algorithm}
\caption{Tensor Head Averaging (THA)}\label{alg:tha}
\begin{algorithmic}[1]
\Require Input series \(\{\mathbf{u}(t)\}_{t=1}^N \), Full system output \(Y\) 
\Require System parameters (Memory \(M\), Degree \(d\), Input dimension \(p\))
\Require Ensemble size \(K\) Subset selection strategy \(\mathcal{S}\).
\Require MVMALS Hyperparameters (Tolerance \(\epsilon\), Max Rank \(\rho_{\max}\)).

\vspace{2mm}

\Function{SelectSubsets}{$\mathcal{S}, K, p$}
    \State \(\{S_k\}_{k=1}^K \leftarrow \text{ApplyStrategy}(\mathcal{S}, K, p)\) \Comment{Determine \(K\) subsets}
    \State \Return \(\{S_k\}_{k=1}^K\)
\EndFunction

\vspace{4mm}
    
\Function{TrainHeads}{$ \{\mathbf{u}(t)\}_{train}, Y_{train}, M, d, K, \{S_k\}_{k=1}^K, \epsilon, \rho_{max}$ }
    \State Initialize ensemble \(\mathcal{E} \leftarrow \emptyset\).
    \For{\(k=1\) to \(K\)} 
        \State \(\{\mathbf{u}^{(k)}(t)\} \leftarrow \mathbf{u}(t)[S_k]\) for \(t\) \(\in\) \text{train} \Comment{ Restrict inputs to the subset \(S_k\)}

        \State \(\mathcal{B}^{(k)} \leftarrow \text{MVMALS}(\{\mathbf{u}^{(k)}(t)\}, Y_{train}, M, d, \epsilon, \rho_{\max})\) 
        
        \State \(\mathcal{E} \leftarrow \mathcal{E} \cup \{\mathcal{B}^{(k)}\}\) \Comment{Store model in TN format}
    \EndFor
    \State \Return \(\mathcal{E}\) \Comment{Ensemble of localized TN models}
\EndFunction

\vspace{4mm}

\Function{OptimizeWeights}{$\{\mathbf{u}(t)\}_{val}, Y_{val}, \mathcal{E}, \{S_k\}$}
    \For{\(k=1\) to \(K\)}
        \State Retrieve \(\mathcal{B}^{(k)}\) from \(\mathcal{E}\).
        \State \(\{\mathbf{u}^{(k)}(t)\} \leftarrow \mathbf{u}(t)[S_k]\) for \(t\) \(\in\) \text{val}
        \State \(\widehat{Y}^{(k)}_{val} \leftarrow \text{SimulateTNVS}(\mathcal{B}^{(k)}, \{\mathbf{u}^{(k)}(t)\})\) \Comment{Simulate on validation data}
    \EndFor
    \State \(\{\omega_k\}_{k=1}^K \leftarrow \arg\min_{\omega} \| Y_{val} - \sum_k \omega_k \widehat{Y}^{(k)}_{val} \|_F^2 \quad \text{s.t. } \sum \omega_k = 1, \omega_k \ge 0\). \Comment{Example strategy}
    \State \Return \(\{\omega_k\}_{k=1}^K\)
    
\EndFunction

\vspace{4mm}

\Function{THA\_Predict}{$\{\mathbf{u}(t)\}_{test}, \mathcal{E}, \{\omega_k\}, \{S_k\}$}
    \State Initialize \(\widehat{Y}^{\mathrm{THA}} \leftarrow \mathbf{0}\).
    \For{\(k=1\) to \(K\)} 
        \State Retrieve \(\mathcal{B}^{(k)}\) from \(\mathcal{E}\).
        \State \(\mathbf{u}^{(k)}(t) \leftarrow \mathbf{u}(t)[S_k]\). for \(t\) \(\in\) \text{test}
        \State \(\widehat{Y}^{(k)} \leftarrow \text{SimulateTNVS}(\mathcal{B}^{(k)}, \{\mathbf{u}^{(k)}(t)\})\)
        \State \(\widehat{Y}^{\mathrm{THA}} \leftarrow \widehat{Y}^{\mathrm{THA}} + \omega_k\,\widehat{Y}^{(k)}\)
    \EndFor
    
    \State \Return \(\widehat{Y}^{\mathrm{THA}}\)
	\EndFunction

\end{algorithmic}
\end{algorithm}

\subsection{Computational Complexity Analysis}
\label{subsec:complexity}

We analyze the computational complexity by comparing the cost of applying MVMALS to the full system versus the cost of the THA ensemble. The complexity of MVMALS is dominated by the least-squares solves and SVD operations performed during the iterative updates of the tensor cores (see Appendix~\ref{app:complexity_details} for a detailed derivation).

For the full system with input dimension \(p\), memory \(M\), and maximal TN-rank \(\rho\), the augmented input dimension is \(m = pM + 1\). Assuming convergence in \(s\) sweeps, the total complexity for the full MVMALS identification is:

\begin{equation} \label{eq:C_total}
\mathcal{C}_{\rm total} =O\bigl(s(d-1)[Nlm^{4}\rho^{4} + m^{6}\rho^{6} + m^{3}\rho^{3}]\bigr).
\end{equation}

The complexity exhibits a high-order polynomial dependence on \(p\), since \(m=O(pM)\). The dominant term scales as \(O(p^6 M^6)\), which represents the primary computational bottleneck.

The THA algorithm addresses this by decomposing the task into \(K\) smaller problems. With a uniform subset size \(k \ll p\) and uniform complexity bounds (\(s_{\max}, \rho_{\max}\)), the total complexity for THA is:

    \[
        \mathcal C_{\text{THA}} = O\bigl(K s_{\max} (d-1)[Nl(kM)^4\rho_{\max}^4 + (kM)^6\rho_{\max}^6 + (kM)^3\rho_{\max}^3]\bigr).
    \]

The \(O(p^6)\) dependence is replaced by \(O(K k^6)\). This allows THA to tackle identification problems involving a large number of inputs \(p\) that would be intractable using the full MVMALS approach.

\section{Observable Finite-Sample Error Bounds}
\label{sec:error_bounds}

In this section, we analyze the approximation error between the THA ensemble and the full MVMALS model. We focus on observable quantities derived from a finite dataset and the identified parameters, without relying on assumptions about the underlying data generation process or global optimality of MVMALS solutions. The purpose of this is to decompose the total error into components which elucidate the effects of incomplete coverage and the optimization behavior, where baking occurs. 

\subsection{The Truncated THA Baseline}

To isolate the effect of the optimization process, where the parameters of the subset models are adjusted to compensate for omitted dynamics, we introduce a baseline model that represents the ensemble without this adjustment. This baseline which we call the truncated THA model, is constructed by directly applying the subset masks to the full MVMALS model parameters. 

\begin{definition}[Truncated THA Model]

Let \(B^\star\) be the parameters of the full MVMALS model. The truncated THA model prediction, \(\widehat{Y}^{\mathrm{Trunc}} \in \mathbb{R}^{N\times l}\), is defined as the weighted average of the predictions generated by the truncated coefficients of the full model for each subset \(S_k\):
    \[
        \widehat{Y}^{\mathrm{Trunc}} \coloneqq \sum_{k=1}^K \omega_k \big(U P_{S_k} B^\star\big).
    \]
    
By linearity and the definition of the Coverage Operator \(C_\omega = \sum_{k=1}^K \omega_k P_{S_k}\), this simplifies to \(\widehat{Y}^{\mathrm{Trunc}} = U C_\omega B^\star.\)
\end{definition}

The truncated THA model represents the prediction that would be achieved if the THA heads were simply the corresponding components of the full model, without any additional optimization.

\subsection{Error Decomposition Components}

We now decompose the total error \(\widehat{Y}^{\mathrm{THA}} - \widehat{Y}^{\mathrm{full}}\) relative to the truncated THA baseline, using the empirical Frobenius norm \(\|\cdot\|_{F,N}\) and its associated inner product \(\langle \cdot, \cdot \rangle_{F,N}\).

The first source of error stems from the incomplete representation of the full system dynamics due to the subset strategy. Here, The term ``tail" is used analogously to the ensemble heads. While the heads collectively capture the dynamics represented by \(C_\omega B^\star\), the tail \(T\) then represents the complementary dynamics omitted from the ensemble.

\begin{definition}[Tail and Empirical Coverage Bias]
The tail \(T \in \mathbb{R}^{Q\times l}\) is the matrix of under-represented coefficients of the full model \(B^\star\) under the THA coverage \(C_\omega\):

    \[
        T\ \coloneqq (I-C_\omega)B^\star.
    \]

\end{definition}

\noindent The magnitude \(\|T\|_F\) quantifies the extent of the parameters omitted or down-weighted by the ensemble. If the coverage improves (e.g., by increasing \(K\)) such that the new coverage \(C_{\omega^+}\) satisfies \(C_{\omega^+} \succeq C_\omega\), then \(\|T\|\) decreases monotonically.

The impact of the tail on the prediction error is captured by the empirical coverage bias.

\begin{definition}[Empirical Coverage Bias]
The Empirical Coverage Bias \(\Delta_{\text{Bias}} \in \mathbb{R}^{N\times l}\) is the error of the Truncated THA model compared to the full MVMALS model.

\[
    \Delta_{\text{Bias}} \coloneqq \widehat{Y}^{\mathrm{Trunc}} - \widehat{Y}^{\mathrm{full}} = U(C_\omega-I)B^\star = -UT.
\]

\end{definition}

%\subsubsection{Baking Adjustment and Alignment}

The second component of the error relates to the deviation of the optimized THA model from the truncated baseline.

\begin{definition}[Empirical Baking Adjustment]
The Empirical Baking Adjustment \(\Delta_{\text{Bake}} \in \mathbb{R}^{N\times l}\) is the difference between the actual THA prediction and the Truncated THA prediction:
    \[
        \Delta_{\text{Bake}} \coloneqq \widehat{Y}^{\mathrm{THA}} - \widehat{Y}^{\mathrm{Trunc}}.
    \]
\end{definition}

The magnitude \(\|\Delta_{\text{Bake}}\|_{F,N}\) measures the aggregate effect of the parameter adjustments made during the MVMALS optimization of the subset heads.

The optimization process aims to minimize prediction error, which incentivizes \(\Delta_{\text{Bake}}\) to compensate for \(\Delta_{\text{Bias}}\). Specifically, the optimization drives \(\Delta_{\text{Bake}}\) towards the required correction \(UT = -\Delta_{\text{Bias}}\). We quantify the effectiveness of this compensation through the concept of alignment.

\begin{definition}[Baking Alignment]
The Baking Alignment \(C_{\text{Align}}\) measures the inner product between the Empirical Baking Adjustment (\(\Delta_{\text{Bake}}\)) and the target correction (\(-\Delta_{\text{Bias}}\)), using the empirical Frobenius inner product:

    \[
        C_{\text{Align}} \coloneqq \langle \Delta_{\text{Bake}}, -\Delta_{\text{Bias}} \rangle_{F,N} = \langle \Delta_{\text{Bake}}, UT \rangle_{F,N}.
    \]
\end{definition}

A positive value of \(C_{\text{Align}}\) indicates that the baking adjustment is successfully counteracting the coverage bias, thereby reducing the total error.

%\subsubsection{Observable Estimation Gaps}

Finally, we define a measure that captures the deviation between the parameters of the optimized heads and the truncated full model, as observed through their predictions.

\begin{definition}[Observable Truncation Gap (\(\varepsilon'_k\))]
For each head \(k\), the observable truncation gap \(\varepsilon'_k \ge 0\) measures the empirical norm of the difference between the prediction of the trained head \(\widehat{B}^{(k)}\) and the prediction of the truncated full model \(P_{S_k}B^\star\):
    \[
        \varepsilon'_k \coloneqq \Big\|U \widehat{B}^{(k)}\ - U P_{S_k}B^\star\Big\|_{F,N}.
    \]
\end{definition}

This quantity bounds the magnitude of the baking adjustment. The Empirical Baking Adjustment can be expressed as a weighted average of the individual head deviations \(E'_k = U\widehat{B}^{(k)} - U P_{S_k} B^\star\):

    \[
        \Delta_{\text{Bake}} = \sum_{k=1}^K \omega_k E'_k.
    \]
    
By the convexity of the norm and Jensen's inequality, we have the bound:

    \[
        \|\Delta_{\text{Bake}}\|_{F,N}^2 \le \left(\sum_{k=1}^K \omega_k \varepsilon'_k\right)^2 \le \sum_{k=1}^K \omega_k (\varepsilon'_k)^2.
    \]
    
This relationship establishes that the magnitude of the aggregate baking adjustment is bounded by the mean squared observable truncation gaps.

\subsection{Main Result: Observable Bounds}

We now present the main theoretical results concerning the approximation error of the THA algorithm. We begin with an exact geometric decomposition of the squared error, which explicitly quantifies the contribution of the baking mechanism. Subsequently, we derive two observable upper bounds on the approximation error.

First, we present the exact geometric decomposition, which follows from the structure of the inner product space defined by the empirical Frobenius norm.

\begin{theorem}[Observable Finite-Sample THA Error Decomposition with Baking Gain]
\label{thm:geometric_decomposition}
Let \(\widehat{Y}^{\mathrm{full}}\) be the prediction of the full MVMALS model, and \(\widehat{Y}^{\mathrm{THA}}\) be the prediction of the THA ensemble. The squared normalized empirical Frobenius error between these predictions is given exactly by the following geometric decomposition:

    \[
        \big\|\widehat{Y}^{\mathrm{THA}}-\widehat{Y}^{\mathrm{full}}\big\|_{F,N}^2 = \underbrace{\|\Delta_{\text{Bias}}\|_{F,N}^2}_{\textbf{Coverage Bias (Truncated Error)}} + \underbrace{\|\Delta_{\text{Bake}}\|_{F,N}^2}_{\textbf{Baking Adjustment Magnitude}} - \underbrace{2 C_{\text{Align}}}_{\textbf{Baking Gain}}.
    \]
    
All terms in this decomposition are observable.
\end{theorem}

\begin{proof}[Proof Sketch]

The decomposition relies on the geometric structure of the empirical Frobenius norm. We decompose the total error relative to the Truncated THA baseline:

\[
    \widehat{Y}^{\mathrm{THA}}-\widehat{Y}^{\mathrm{full}} = (\widehat{Y}^{\mathrm{THA}} - \widehat{Y}^{\mathrm{Trunc}}) + (\widehat{Y}^{\mathrm{Trunc}} - \widehat{Y}^{\mathrm{full}}) = \Delta_{\text{Bake}} + \Delta_{\text{Bias}}.
\]

Expanding the squared norm \(\|\Delta_{\text{Bake}} + \Delta_{\text{Bias}}\|_{F,N}^2\) yields the terms \(\|\Delta_{\text{Bake}}\|_{F,N}^2 + \|\Delta_{\text{Bias}}\|_{F,N}^2 + 2 \langle \Delta_{\text{Bake}}, \Delta_{\text{Bias}} \rangle_{F,N}\). The cross-term is exactly \(-2 C_{\text{Align}}\) by definition. See Appendix~\ref{app:proof_thm1} for the full derivation.

\end{proof}
%For the proof, please see~\ref{app:proof_thm1}.

\noindent This decomposition provides some insight into the optimization dynamics. When baking is effective, the optimization ensures \(C_{\text{Align}} > 0\). The term \(2C_{\text{Align}}\) quantifies the reduction in total error achieved by the optimization process, demonstrating how the optimization tightens the error measure.

Next, we derive an upper bound based on the triangle inequality, which provides a clear separation between the error due to coverage and the error due to estimation mismatch.

\begin{theorem}[Observable Finite-Sample THA Bound for MIMO Systems]
\label{thm:triangle_bound}

Let \(\widehat{Y}^{\mathrm{full}} = UB^\star\) and \(\widehat{Y}^{\mathrm{THA}}\) be the predictions of the full MVMALS model and the THA ensemble, respectively. The normalized empirical Frobenius error between these predictions is bounded as follows:

    \[
        \big\|\widehat{Y}^{\mathrm{THA}}-\widehat{Y}^{\mathrm{full}}\big\|_{F,N} \le \underbrace{\sqrt{\lambda_{\max}(G)}\big\|T\big\|_F}_{\textbf{Coverage Bias Bound}} + \underbrace{\Big(\sum_{k=1}^K \omega_k (\varepsilon'_k)^2\Big)^{1/2}}_{\textbf{Aggregated Estimation Gap}}.
    \]

\end{theorem}

\begin{proof}[Proof Sketch]
The proof utilizes the triangle inequality on the error decomposition: \(\|\widehat{Y}^{\mathrm{THA}}-\widehat{Y}^{\mathrm{full}}\|_{F,N} \le \|\Delta_{\text{Bake}}\|_{F,N} + \|\Delta_{\text{Bias}}\|_{F,N}\).
The bias term \(\|\Delta_{\text{Bias}}\|_{F,N} = \|UT\|_{F,N}\) is bounded using the Rayleigh quotient, yielding the bound \(\sqrt{\lambda_{\max}(G)}\|T\|_F\).
The baking adjustment term \(\|\Delta_{\text{Bake}}\|_{F,N}\) is bounded using Jensen's inequality applied to the weighted average of the individual head deviations, resulting in the Aggregated Estimation Gap term. See Appendix~\ref{app:proof_thm2}.
\end{proof}

We can obtain a tighter bound by combining the geometric decomposition with the bound on the estimation gap.

\begin{theorem}[Improved Bound With Baking Term]

\label{thm:improved_bound} 
Let \(\widehat{Y}^{\mathrm{full}}\) and \(\widehat{Y}^{\mathrm{THA}}\) be the predictions of the full MVMALS and the THA ensemble, respectively. The squared normalized empirical Frobenius error is bounded as follows:

    \[
    \big\|\widehat{Y}^{\mathrm{THA}}-\widehat{Y}^{\mathrm{full}}\big\|_{F,N}^2 \le \underbrace{\|\Delta_{\text{Bias}}\|_{F,N}^2}_{\textbf{Empirical Coverage Bias}} + \underbrace{\sum_{k=1}^K \omega_k (\varepsilon'_k)^2}_{\textbf{Aggregated Estimation Gap (MSE)}} - \underbrace{2\, C_{\text{Align}}}_{\textbf{Baking Gain}}.
    \]

\end{theorem}

\begin{proof}[Proof Sketch]
This bound follows directly from the exact geometric decomposition in Theorem~\ref{thm:geometric_decomposition} by substituting the Jensen's inequality bound for \(\|\Delta_{\text{Bake}}\|_{F,N}^2\) derived in this section. See Appendix~\ref{app:proof_thm3}.
\end{proof}

This improved bound is strictly tighter than the squared form of Theorem~\ref{thm:triangle_bound} whenever the Baking Gain \(C_{\text{Align}}\) is positive (and the Jensen's inequality gap is small). It provides a direct measure of how the optimization process improves the approximation beyond simple truncation.

\section{Optimization Dynamics and Baking}
\label{sec:optimization_main}

The error decomposition in Theorem~\ref{thm:geometric_decomposition} quantifies the Baking Gain ($C_{\text{Align}}$), demonstrating that optimization improves the approximation beyond truncation. We now summarize the analysis of the MVMALS dynamics (detailed in Appendix~\ref{sec:optimization_baking}) to explain the mechanism behind this gain and establish that this behavior is almost surely to occur.

Consider the optimization of a single head $k$. We analyze the optimization landscape relative to a baseline $B_{\text{proj}}^{(k)}$, defined as the best approximation of the truncated full model parameters $P_{S_k} B^\star$ within the feasible low-rank TN set. This baseline represents the absence of baking.

The incentive to deviate from $B_{\text{proj}}^{(k)}$ stems from the correlation between the included features and the dynamics generated by the omitted parameters.

\begin{definition}[Global Correlation]
The Global Correlation $C$ quantifies the empirical relationship between the included features $U_k = U P_{S_k}$ and the omitted dynamics $R_{\text{omitted}}$ (the prediction generated by the excluded parameters):
\[
C \coloneqq U_k^\top R_{\text{omitted}}.
\]
\end{definition}

If $C \neq 0$, a gradient exists in the ambient space that incentivizes adjusting parameters to compensate for $R_{\text{omitted}}$. However, the optimization is constrained to the TN manifold. The insight lies in the geometry of this manifold.

\begin{theorem}[Geometric Condition for Baking Incentive]
\label{thm:geometric_incentive_main}
Under standard regularity conditions for the TN manifold (Assumption~\ref{ass:regularity} in the Appendix), a gradient incentive for baking exists if and only if the Global Correlation $C$ is not orthogonal to the tangent space of the TN manifold at $B_{\text{proj}}^{(k)}$.
\end{theorem}
\begin{proof}[Proof Sketch]
The gradient of the loss function on the manifold is the projection of the ambient gradient (driven by $C$) onto the tangent space. If $C$ is orthogonal to the tangent space, this projection is zero. If the projection is non-zero, $B_{\text{proj}}^{(k)}$ is not a stationary point. The MVMALS algorithm, characterized by monotonic convergence, acts upon this gradient to reduce the loss, moving the estimate away from $B_{\text{proj}}^{(k)}$. See Appendix~\ref{sec:optimization_baking} Theorems 5 \& 6.
\end{proof}

This theorem establishes that the low-rank structure could potentially prevent baking if the correlation aligns perfectly with directions inaccessible to the optimization. However, we show this scenario has measure zero in the appendix under Proposition \ref{app:full_measure_arg}.

\section{Conclusion}
\label{sec:conclusion}

The identification of high-dimensional nonlinear dynamical systems has long been constrained by the complexity inherent in the Volterra series. While TN methods like MVMALS successfully mitigated the exponential complexity arising from system nonlinearity, the high-order polynomial memory complexity dependent on the input dimension \(p\) remained a significant, unresolved bottleneck. In this work, we introduced the THA algorithm and established a framework that directly addresses this challenge. By decomposing the identification task into an ensemble of localized MVMALS models, THA significantly reduces computational complexity for systems with large input dimension, replacing the dependence on the full dimension \(p\) with the much smaller subset size \(k\).

These results substantially expand the applicability of Volterra identification, enabling the analysis of complex systems that were previously intractable. A central contribution of this paper is the derivation of observable, finite-sample error bounds that quantify the approximation performance of the THA ensemble relative to a full MVMALS model, without relying on unrealistic assumptions of convexity or global optimality.

Furthermore, our exact geometric decomposition of the squared error formalizes and quantifies the baking, and we demonstrated how localized models implicitly compensate for omitted dynamics. We have also proven that the optimization process is almost surely incentivized to leverage inter-variable correlations, and this mechanism ensures that the THA ensemble achieves accuracy superior to a simple truncation of the full MVMALS model. This provides a grounded argument for the effectiveness of the ensemble approach and makes THA more than just a practical heuristic.

The potential impact of scalable Volterra identification extends across a diverse set of scientific and engineering disciplines where high-dimensional nonlinear interactions are prevalent.\cite{marmarelis_nonlinear_2004,rugh_nonlinear_1981,zhu_behavioral_2016,favier_nonlinear_2012} The Volterra series as a universal approximator provides an interpretable representation with the ability to also capture memory effects.\cite{schetzen_volterra_1980,boyd_fading_1985} This interpretability can be pivotal for numerous domains where understanding the underlying mechanisms is equally important to predictive accuracy, such as in analysis of physiological systems\cite{marmarelis_analysis_1978,korenberg_identification_1996} or in advanced control systems where stability guarantees can depend on the explicit structure of a model.\cite{doyle_identification_2002,rugh_nonlinear_1981} By rendering the identification of these models tractable for large input spaces, the THA algorithm enables the investigation of the dynamics of large complex systems.

The foundation established in this paper paves the way for both immediate practical application as well as further algorithmic refinement. Future work will focus on empirical validation of the derived bounds and the properties of baking across a variety of high-dimensional benchmarks resembling real world systems. Moreover, developing adaptive strategies for optimal subset selection and weighting represents a critical next step in maximizing the balance between computational efficiency and predictive accuracy. Ultimately, the THA framework offers a powerful and scalable approach for the analysis of sophisticated real-world nonlinear systems previously beyond the reach of existing identification methods.

\newpage
\nocite{*}

\printbibliography % Uncomment when bib file is ready

\newpage

\appendix

\section{Optimization Dynamics and the Incentive for Baking}
\label{sec:optimization_baking}

In this section, we analyze the optimization dynamics of the MVMALS algorithm when applied to a single THA head. We investigate the conditions under which the optimization process drives the subset model parameters away from a simple truncation of the full model. This process, termed baking, allows the subset model to compensate for omitted dynamics by leveraging correlations present in the data. We demonstrate that this incentive is linked to the geometry of the TN manifold and is almost surely a property of the optimization landscape.

\subsection{Setup: Optimization on the TN Manifold}

We analyze the optimization process for a single THA head, indexed by \((k)\), corresponding to a subset of variables \(S_k\). For notational simplicity, we denote the mask operator \(P_{S_k}\) as \(P_k\).

\subsubsection{Data Partitioning and the Optimization Landscape}

We assume the data \(Y\in\mathbb{R}^{N\times l}\) is generated by the full Volterra model \(Y = U B^\star + E\), where \(E\) represents noise and irreducible error. We partition the system based on the subset \(S_k\):

\begin{itemize}
    \item Included features: \(U_k \coloneqq U P_k\).
    \item Omitted features: \(U_{\neg k} \coloneqq U (I-P_k)\).
    \item Truncated parameters: \(B_k^\star \coloneqq P_k B^\star\).
    \item Omitted parameters: \(B_{\neg k}^\star \coloneqq (I-P_k) B^\star\).
\end{itemize}

Since \(P_k\) is a projection matrix (\(P_k^2 = P_k\)), the cross-terms vanish (\(U_k B_{\neg k}^\star = U P_k (I-P_k) B^\star = 0\)). The data model is thus decomposed as:

\begin{equation}
    Y = U_k B_k^\star + U_{\neg k} B_{\neg k}^\star + E.
    \label{eq:partitioned_data_model}
\end{equation}

The optimization for head \(k\) is constrained to the set of tensors whose coefficients lie within the subset space (i.e., \(P_k B = B\)) and have TN ranks bounded by \(\mathbf{r}\). We denote this closed algebraic set as \(\mathcal{T}_{\le \mathbf{r}}(S_k)\). The objective is to minimize the empirical loss:

    \[
        \min_{B \in \mathcal{T}_{\le \mathbf{r}}(S_k)} L(B) \coloneqq \|Y - U_k B\|_F^2.
    \]

To analyze the incentive for baking, we compare the optimized solution to a baseline defined by the projection of the full model parameters onto the feasible set.

\begin{definition}[Feasible TT Projection \(B_{\text{proj}}^{(k)}\)]
The Feasible TT Projection \(B_{\text{proj}}^{(k)}\) is the best approximation of the truncated parameters \(B_k^\star\) within the set \(\mathcal{T}_{\le \mathbf{r}}(S_k)\):

    \[
        B_{\text{proj}}^{(k)} \coloneqq \arg\min_{B \in \mathcal{T}_{\le \mathbf{r}}(S_k)} \|B - B_k^\star\|_F^2.
    \]

\end{definition}

\noindent The existence of \(B_{\text{proj}}^{(k)}\) is guaranteed as we are minimizing a continuous, coercive function over a closed set.

We analyze the optimization dynamics starting from this projection. Let \(\mathcal{U}^{(0)}\) be the TN cores corresponding to \(B^{(0)} = B_{\text{proj}}^{(k)}\) via the multilinear map \(\tau\). The parametrized objective is \(j(\mathcal{U}) = L(\tau(\mathcal{U}))\). The analysis relies on the local geometry of the TN manifold \(\mathcal{T}_{\mathbf{r}}(S_k)\).

\subsubsection{MVMALS Framework and Regularity}

MVMALS utilizes ALS or MALS principles, maintaining specific orthogonalization conditions (the ALS gauge) for numerical stability. The analysis of the gradient requires standard regularity assumptions from TN optimization theory.

\begin{assumption}[Regularity]\label{ass:regularity}
We assume that the TN representation \(\mathcal{U}^{(0)}\) corresponding to \(B_{\text{proj}}^{(k)}\) has exactly rank \(\mathbf{r}\) (i.e., it lies on the smooth manifold \(\mathcal{T}_{\mathbf{r}}(S_k)\)). This ensures that the map \(\tau\) is a submersion at \(\mathcal{U}^{(0)}\), as proven by Rohwedder and Uschmajew (2013, Prop A.3)\cite{rohwedder_local_2013}. Consequently, the image of the Jacobian \(J_\tau(\mathcal{U}^{(0)}) = \tau'(\mathcal{U}^{(0)})\) is exactly the tangent space of the manifold at \(B^{(0)}\), denoted \(T_{B^{(0)}}\). This also implies that the TN environments maintain full rank (Holtz, Rohwedder, \& Schneider, 2012)\cite{holtz_alternating_2012}. 
\end{assumption}

\subsection{MVMALS Dynamics and Truncation Loss}

We analyze a single MVMALS (MALS) update step starting from \(B^{(0)} = B_{\text{proj}}^{(k)}\), assuming \(B^{(0)}\) is in the appropriate orthonormal gauge for the block being updated. This analysis reveals how the algorithm responds to a gradient incentive, considering the constraints imposed by rank truncation.

\begin{theorem}[Optimization Incentive under TT-ALS/MALS]\label{thm:optimization_incentive}
Let \(B^{(0)} = B_{\text{proj}}^{(k)}\). Define the initial residual:

    \[ 
        R_0 \coloneqq Y - U_k B^{(0)}.
    \]
    
Consider a single ALS/MALS block update (indexed by \(i\)) starting from \(B^{(0)}\). Let \(A_{0,(i)}\) be the corresponding local design matrix as defined in Batselier et al., 2017, Theorems 4.1/4.2)\cite{batselier_tensor_2017}. Let \(r_0 = \mathrm{vec}(R_0^\top)\). If the local gradient is non-zero, i.e.,

    \[ 
        A_{0,(i)}^\top r_0 \neq 0, 
    \]
    
then the untruncated least squares (LS) update results in a potential decrease in the objective function:

    \[ 
        \Delta J_{\text{LS}} = \big\|\Pi_{\mathcal{R}(A_{0,(i)})}\,r_0\big\|_2^2 > 0. 
    \]
    
The realized update in MVMALS involves an SVD truncation step. Let \(\Delta z^*\) and \(\Delta z_{\text{trunc}}\) be the vectorized parameter updates before and after truncation, respectively. The realized decrease \(\Delta J_{\text{trunc}}\) satisfies:

    \[ 
        \Delta J_{\text{trunc}} = \Delta J_{\text{LS}} - \underbrace{\|A_{0,(i)}(\Delta z^* - \Delta z_{\text{trunc}})\|_2^2}_{\text{Truncation Loss}}. 
    \]
    
If the truncation loss is smaller than the LS gain, the objective function strictly decreases, moving the estimate away from \(B^{(0)}\).
\end{theorem}
\begin{proof}
The optimization subproblem for the \(i\)-th block minimizes the linearized objective with respect to the update \(\Delta z_{(i)}\):

    \[ 
        J_i(\Delta z_{(i)}) = \|r_0 - A_{0,(i)} \Delta z_{(i)}\|_2^2. 
    \]
    
The initial cost is \(J_i(0) = \|r_0\|_2^2\). The optimal untruncated LS update, \(\Delta z^*\), projects \(r_0\) onto the column space \(\mathcal{R}(A_{0,(i)})\):

    \[ 
        A_{0,(i)} \Delta z^* = \Pi_{\mathcal{R}(A_{0,(i)})} r_0. 
    \]
    
By hypothesis, the gradient \(A_{0,(i)}^\top r_0 \neq 0\), implying the projection is non-zero. The decrease is \(\Delta J_{\text{LS}} = J_i(0) - J_i(\Delta z^*) = \|A_{0,(i)} \Delta z^*\|_2^2 > 0\).

The MVMALS algorithm applies an SVD truncation, yielding the realized update \(\Delta z_{\text{trunc}}\). We analyze the realized cost \(J_i(\Delta z_{\text{trunc}})\):

    \[ 
        J_i(\Delta z_{\text{trunc}}) = \|r_0 - A_{0,(i)} \Delta z_{\text{trunc}}\|_2^2 = \|(r_0 - A_{0,(i)} \Delta z^*) + A_{0,(i)}(\Delta z^* - \Delta z_{\text{trunc}})\|_2^2. 
    \]
    
The first term is the residual of the optimal LS projection and is orthogonal to \(\mathcal{R}(A_{0,(i)})\). The second term lies within \(\mathcal{R}(A_{0,(i)})\). By the Pythagorean theorem:

    \[ 
        J_i(\Delta z_{\text{trunc}}) = \|r_0 - A_{0,(i)} \Delta z^*\|_2^2 + \|A_{0,(i)}(\Delta z^* - \Delta z_{\text{trunc}})\|_2^2. 
    \]
    
The total realized decrease is:

    \begin{align*}
        \Delta J_{\text{trunc}} &= J_i(0) - J_i(\Delta z_{\text{trunc}}) \\
        &= (J_i(0) - J_i(\Delta z^*)) - (J_i(\Delta z_{\text{trunc}}) - J_i(\Delta z^*)) \\
        &= \Delta J_{\text{LS}} - \|A_{0,(i)}(\Delta z^* - \Delta z_{\text{trunc}})\|_2^2.
\end{align*}

If \(\Delta J_{\text{trunc}} > 0\), the algorithm strictly improves the fit. Monotonicity checks within MVMALS ensure that the algorithm proceeds only if the truncation loss does not cause the error to increase significantly.
\end{proof}

\subsection{The Mechanism of Baking: Omitted Variable Correlation}

We now investigate the conditions under which the local gradient is non-zero. This requires analyzing the structure of the initial residual \(R_0\).

\begin{lemma}[Decomposition of the Residual]\label{lem:residual_decomposition}
The residual \(R_0\) at the Feasible TT Projection \(B_{\text{proj}}^{(k)}\) can be decomposed as:
    \[ 
        R_0 = \underbrace{E}_{\text{Noise}} + \underbrace{U_{\neg k} B_{\neg k}^\star}_{R_{\text{omitted}}} + \underbrace{U_k (B_k^\star - B_{\text{proj}}^{(k)})}_{R_{\text{proj\_err}}}. 
    \]
\end{lemma}

\begin{proof}
We substitute the partitioned data model \eqref{eq:partitioned_data_model} into the definition of the residual \(R_0 = Y - U_k B_{\text{proj}}^{(k)}\):

    \[ 
        R_0 = (U_k B_k^\star + U_{\neg k} B_{\neg k}^\star + E) - U_k B_{\text{proj}}^{(k)}. 
    \]
    
Rearranging the terms yields the decomposition:

    \[ 
        R_0 = E + U_{\neg k} B_{\neg k}^\star + U_k (B_k^\star - B_{\text{proj}}^{(k)}). 
    \]
\end{proof}

\begin{definition}[Global Correlation C]
The global correlation \(C \in \mathbb{R}^{Q\times l}\) (represented in matricized form) quantifies the influence of the omitted dynamics \(R_{\text{omitted}}\) on the included features:

    \[ 
        C \coloneqq U_k^\top R_{\text{omitted}}. 
    \]
\end{definition}

The term \(R_{\text{omitted}}\) represents the prediction generated by the dynamics excluded from the subset. Baking is the optimization's response to this term, driven by its correlation with the included features.

\subsection{The Geometric Condition for Baking}

In the unconstrained setting (full rank), a non-zero correlation \(C \neq 0\) guarantees a non-zero gradient of the loss function \(L(B)\) at \(B^{(0)}\). However, in the low-rank TN setting, the optimization occurs on the manifold \(\mathcal{T}_{\mathbf{r}}(S_k)\). We must establish the link between the global correlation \(C\) and the gradient of the parametrized objective \(j(\mathcal{U})\) in the TN core space.

The component of the gradient induced by the omitted dynamics is given by the chain rule and the adjoint of the Jacobian \(J_\tau(\mathcal{U}^{(0)})^*\):

    \begin{align*}
        \nabla j_{\text{omitted}}(\mathcal{U}^{(0)}) &= J_\tau(\mathcal{U}^{(0)})^* [\nabla_B L_{\text{omitted}}(B^{(0)})] \\
        &= -2 J_\tau(\mathcal{U}^{(0)})^* [U_k^\top R_{\text{omitted}}] = -2 J_\tau(\mathcal{U}^{(0)})^* [C].
    \end{align*}

The condition \(\nabla j_{\text{omitted}}(\mathcal{U}^{(0)}) = 0\) is equivalent to \(C \in \text{Ker}(J_\tau(\mathcal{U}^{(0)})^*)\). The kernel of the adjoint is the orthogonal complement of the image: \(\text{Ker}(J_\tau(\mathcal{U}^{(0)})^*) = (\text{Im}(J_\tau(\mathcal{U}^{(0)})))^\perp\). Under the Regularity Assumption~\ref{ass:regularity}, \(\text{Im}(J_\tau(\mathcal{U}^{(0)})) = T_{B^{(0)}}\). Therefore:

    \[ 
        \nabla j_{\text{omitted}}(\mathcal{U}^{(0)}) = 0 \iff C \in (T_{B^{(0)}})^\perp. 
    \]

If the TN ranks \(\mathbf{r}\) are constrained, \(T_{B^{(0)}}\) is a proper subspace. It is geometrically possible for a non-zero correlation \(C\) to be orthogonal to the tangent space, in which case it exerts no driving force on the manifold.

% (Analysis: Gradient is zero iff C is orthogonal to the tangent space T_B.)
\begin{lemma}[Condition for Non-Zero Omitted Dynamics Gradient]
\label{lem:geometric_condition}
Assume the Regularity Assumption~\ref{ass:regularity} holds at \(\mathcal{U}^{(0)}\). The contribution of the omitted dynamics to the MVMALS local gradient is non-zero for at least one optimization block \(i\) if and only if the global correlation \(C = U_k^\top R_{\text{omitted}}\) is not orthogonal to the tangent space \(T_{B^{(0)}}\).
\end{lemma}

\begin{proof}
The MVMALS local gradients are the block components of the full gradient \(\nabla j_{\text{omitted}}(\mathcal{U}^{(0)})\). Thus, a local gradient is non-zero if and only if the full gradient is non-zero. As established above, \(\nabla j_{\text{omitted}}(\mathcal{U}^{(0)}) \neq 0\) if and only if \(C \notin (T_{B^{(0)}})^\perp\).
\end{proof}

\subsection{Sufficient Conditions and Generality}

We now combine these elements to establish the sufficient conditions for the optimization to induce baking.

\begin{theorem}[Sufficient Condition for Baking Incentive]
\label{thm:sufficient_condition}
If the global correlation \(C = U_k^\top R_{\text{omitted}}\) is not orthogonal to the tangent space \(T_{B^{(0)}}\) at \(B^{(0)} = B_{\text{proj}}^{(k)}\), and the Regularity Assumption~\ref{ass:regularity} holds, this correlation provides a gradient direction away from \(B_{\text{proj}}^{(k)}\). This incentivizes baking, provided this effect is not perfectly canceled by the noise \(E\) and the TT projection error \(R_{\text{proj\_err}}\), and that the subsequent SVD truncation (in MVMALS) does not nullify the improvement.
\end{theorem}

\begin{proof}
We analyze the stationarity of the parametrized objective function \(j(\mathcal{U}) = L(\tau(\mathcal{U}))\) at the point \(\mathcal{U}^{(0)}\) where \(\tau(\mathcal{U}^{(0)}) = B^{(0)} = B_{\text{proj}}^{(k)}\). An optimization incentive exists if \(\mathcal{U}^{(0)}\) is not a stationary point, and the MVMALS algorithm realizes a descent step. 

The gradient of the objective function in the parameter space is determined by the pullback of the gradient in the tensor space via the adjoint of the Jacobian: 

\[
    \nabla j(\mathcal{U}^{(0)}) = J_\tau(\mathcal{U}^{(0)})^* [\nabla_B L(B^{(0)})] = -2 J_\tau(\mathcal{U}^{(0)})^* [U_k^\top R_0].
\]

Using the decomposition of the initial residual \(R_0\) established in Lemma~\ref{lem:residual_decomposition}, we decompose the gradient into its components: 

\begin{align*} 
    \nabla j(\mathcal{U}^{(0)}) &= -2 J_\tau(\mathcal{U}^{(0)})^* [U_k^\top R_{\text{omitted}}] -2 J_\tau(\mathcal{U}^{(0)})^* [U_k^\top R_{\text{proj\_err}}] -2 J_\tau(\mathcal{U}^{(0)})^* [U_k^\top E] \\ &= \nabla j_{\text{omitted}}(\mathcal{U}^{(0)}) + \nabla j_{\text{proj}}(\mathcal{U}^{(0)}) + \nabla j_{\text{noise}}(\mathcal{U}^{(0)}). 
\end{align*}

The component induced by the omitted dynamics is specifically \(\nabla j_{\text{omitted}}(\mathcal{U}^{(0)}) = -2 J_\tau(\mathcal{U}^{(0)})^* [C]\) where \(C = U_k^\top R_{\text{omitted}}\) is the global correlation. 

Under the regularity assumption ~\ref{ass:regularity}, the map \(\tau\) is a submersion at \(\mathcal{U}^{(0)}\) which implies \(\text{Im}(J_\tau(\mathcal{U}^{(0)})) = T_{B^{(0)}}\). As discussed in Lemma~\ref{lem:geometric_condition}, the condition \(C \notin (T_{B^{(0)}})^\perp\) is equivalent to \(C \notin \text{Ker}(J_\tau(\mathcal{U}^{(0)})^*)\). Hence, the contribution of the omitted dynamics to the gradient is strictly non-zero:
\[
    \nabla j_{\text{omitted}}(\mathcal{U}^{(0)}) \neq 0.
\]

As long as the force is not exactly canceled by the contributions from the projection error and noise, it follows that the total gradient \(\nabla j(\mathcal{U}^{(0)})\) is non-zero, and thus \(\mathcal{U}^{(0)}\) is not a block stationary point of \(j(\mathcal{U})\). 

Because the total gradient is non-zero, then at least one block component of the gradient must also be non-zero. The MVMALS algorithm is characterized by its monotonic convergence to a block-stationary point\cite{rohwedder_local_2013, batselier_tensor_2017}. As analyzed in Theorem~\ref{thm:optimization_incentive}, this implies that the untruncated least-squares update for this block yields a potential strict decrease \(\Delta J_{\text{LS}} > 0\). The algorithm will execute the update which results in a strict reduction of the objective function value. 

This necessarily moves the update away from \(B_{\text{proj}}^{(k)}\). As this movement is driven, at least in part by \(\nabla j_{\text{omitted}}\) originating from the correlation \(C\), the optimization process actively adjusts the parameters to partially account for the omitted dynamics. This optimization-driven deviation from the simple projection constitutes the mechanism of baking. 
\end{proof}

We now demonstrate that the geometric condition required for baking holds with probability 1.

\begin{proposition}[Condition Preventing Baking is Measure Zero]
\label{app:full_measure_arg}
Let the ambient space of coefficients be \(V = \mathbb{R}^{Q\times l}\). If the rank constraints are active, the tangent space \(T_{B^{(0)}}\) and its orthogonal complement \(T_{B^{(0)}}^\perp\) are proper subspaces of \(V\). Assuming the probability distribution of the correlation vector \(C\) (derived from the input data \(U\) and system parameters \(B^\star\)) is absolutely continuous with respect to the Lebesgue measure on \(V\), the probability that baking is prevented despite \(C\neq 0\) is zero:

    \[ 
        P(C \in T_{B^{(0)}}^\perp) = 0. 
    \]
\end{proposition}

\begin{proof}
The assumption of absolute continuity implies that the process generating \(C\) is not deterministically constrained to a lower-dimensional structure (e.g., due to continuous distributions of inputs or parameters). A fundamental property of the Lebesgue measure in \(\mathbb{R}^{\dim(V)}\) is that any proper linear subspace has measure zero. Since \(T_{B^{(0)}}^\perp\) is a proper subspace, its Lebesgue measure is zero. By absolute continuity, the probability of a random vector falling into a set with zero Lebesgue measure is zero.
\end{proof}

The analysis confirms that the MVMALS algorithm, guaranteed to converge to a block-stationary point, systematically exploits correlations between included and omitted dynamics. The simple projection \(B_{\text{proj}}^{(k)}\) (which optimizes parameter distance) is generally \emph{not} block-stationary (which optimizes prediction error) if two conditions are met:

\begin{enumerate}
    \item Significant correlation exists (\(C\neq 0\)).
    \item (The Geometric Condition) This correlation \(C\) is not orthogonal to the tangent space of the TN manifold at \(B_{\text{proj}}^{(k)}\).
\end{enumerate}

With these conditions holding almost surely, the algorithm iterates away from \(B_{\text{proj}}^{(k)}\) to a point \(B_{\text{baked}}^{(k)}\) where \(L(B_{\text{baked}}^{(k)}) < L(B_{\text{proj}}^{(k)})\). This improvement in empirical fit is achieved by baking the influence of the omitted, correlated dynamics into the subset head.

\section{Tensor Network Details}
\label{app:tn_details}

\subsection{TN Contraction Formula}
The multilinear map \(\tau\) introduced in Definition 6 constructs the full tensor \(\mathcal{B} = \tau(\mathcal{U})\) element-wise via the following contraction:

    \[
    \mathcal{B}(j, i_1, \dots, i_d) = \sum_{\alpha_1=1}^{r_1} \cdots \sum_{\alpha_{d-1}=1}^{r_{d-1}} V^{(1)}(j, i_1, \alpha_1) V^{(2)}(\alpha_1, i_2, \alpha_2) \cdots V^{(d)}(\alpha_{d-1}, i_d, 1).
    \]
    
Here, the indices \(i_k \in \{1, \dots, q\}\) correspond to the input features, and \(j \in \{1, \dots, l\}\) corresponds to the output channel. The map \(\tau\) is generally non-injective. Equivalent core representations of the same tensor form orbits under the action of a Lie group of invertible transformations (gauge transformations).

\subsection{Remarks on the MIMO Tensor Network Structure}
\label{app:mimo_tn_structure}

The standard TT format for which the foundational theories on manifold structure and ALS/MALS convergence are typically derived\cite{holtz_manifolds_2012}\cite{rohwedder_local_2013}, assumes boundary ranks $r_0=r_d=1$. The MIMO TN structure employed in this paper utilizes $r_0=l$ (the number of outputs) and $r_d=1$. We wish to clarify for the reader why the established theoretical results remain applicable in this setting.

While the dimension of the parameter space and the resulting manifold $\mathcal{T}_{\mathbf{r}}$ depend on $l$, the  structure governing the optimization dynamics is the gauge group $\mathcal{G} = GL(r_1) \times \cdots \times GL(r_{d-1})$ which acts on the internal ranks. This remains unchanged. The boundary ranks do not participate in the gauge transformations that leave the tensor invariant.

Furthermore, the MIMO TN optimization problem is isomorphic to a standard TT optimization problem. The tensor $\mathcal{B} \in \mathbb{R}^{l \times q \times \cdots \times q}$ can be reshaped into a $d$-order tensor $\mathcal{B}' \in \mathbb{R}^{(lq) \times q \times \cdots \times q}$. The MIMO TN representation with cores $(V^{(1)}, \dots, V^{(d)})$ corresponds exactly to a standard TT representation of $\mathcal{B}'$ with cores $(W^{(1)}, \dots, W^{(d)})$, where $W^{(1)} \in \mathbb{R}^{1 \times (lq) \times r_1}$ is the reshaping of $V^{(1)} \in \mathbb{R}^{l \times q \times r_1}$, and $W^{(k)}=V^{(k)}$ for $k\ge 2$.

This isomorphism preserves the parameter space structure, the gauge group action, and the objective function landscape (as the reshaping is isometric with respect to the Frobenius norm). Consequently, the theoretical results concerning the manifold structure (e.g., the submersion property of $\tau$) and the convergence of MVMALS apply directly to the MIMO setting presented here.

\section{Detailed Computational Complexity Analysis}
\label{app:complexity_details}

This section provides a detailed derivation of the computational complexity for the MVMALS algorithm, which forms the basis for the complexity analysis presented in Section~\ref{subsec:complexity}.

\subsection{Complexity of Full MVMALS}

For the full system, the dimension of the augmented input vector is \(m = pM + 1\). We define the maximal TT-rank as \(\rho\). The MVMALS algorithm iteratively updates the TN representation through sweeps. Each sweep involves updating \((d-1)\) ``super-cores" \(V^{(k,k+1)}\in\mathbb{R}^{r_{k-1}\times m^{2}\times r_{k+1}}\). The update of each super-core requires two main computational tasks:

\begin{enumerate}
 \item \textbf{Least Squares (LS) Solve:} Solving a reduced linear system \(U_{k,k+1} \operatorname{vec}(V^{(k,k+1)}) = \operatorname{vec}(Y^\top)\). The local design matrix is \(U_{k,k+1}\in\mathbb{R}^{lN\times r_{k-1} m^{2} r_{k+1}}\). Computing the minimal-norm solution requires \(O(Nl(r_{k-1}m^{2}r_{k+1})^{2}+(r_{k-1}m^{2}r_{k+1})^{3})\) flops.
 
 \item \textbf{Rank-revealing SVD:} Reshaping the solution to a matrix \(V_{k,k+1}\in\mathbb{R}^{r_{k-1}m\times mr_{k+1}}\) and computing its thin SVD for orthogonalization and rank adaptation. This requires \(O(r_{k-1}^{2}m^{3}r_{k+1} + (r_{k-1}m^{3}r_{k+1})^{2} + (m^3r^3_{k+1}))\) flops.
 
\end{enumerate}

To derive the per-sweep complexity, \(\mathcal{C}_{\text{sweep}}\), we sum these costs over the \((d-1)\) super-cores. We approximate all ranks by the maximal rank \(\rho\) to obtain an asymptotic estimate:

    \[
    \mathcal{C}_{\rm sweep} =O\bigl((d-1)[Nlm^{4}\rho^{4} + m^{6}\rho^{6} + m^{3}\rho^{3} + \rho^4m^6]\bigr).
    \]

Assuming convergence in \(s\) sweeps, the total complexity for the full MVMALS identification, retaining the dominant terms, leads to Eq.~\eqref{eq:C_total}.

\subsection{Complexity of THA}

The THA algorithm decomposes the identification task into \(K\) smaller problems. Let \(k_i\) be the size of the \(i\)-th subset, \(m_{k_i} = k_iM+1\), and let \(s_{k_i}, \rho_{k_i}\) be the number of sweeps and maximal rank for that subset. The complexity of identifying a single subset model is:
    \[
    \mathcal C_{\text{MVMALS}}^{(k_i)} =O\bigl(s_{k_i}(d-1)[Nlm_{k_i}^{4}\rho_{k_i}^{4} + m_{k_i}^{6}\rho_{k_i}^{6} + m_{k_i}^{3}\rho_{k_i}^{3}]\bigr).
    \]

Assuming a uniform subset size \(k\) and uniform complexity bounds (\(s_{\max}, \rho_{\max}\)), the total complexity \(\mathcal{C}_{\text{THA}} = \sum_i \mathcal C_{\text{MVMALS}}^{(k_i)}\) simplifies to the expression given in Section~\ref{subsec:complexity}.

    \[
        \mathcal C_{\text{THA}} = O\bigl(K s_{\max} (d-1)[Nl(kM)^4\rho_{\max}^4 + (kM)^6\rho_{\max}^6 + (kM)^3\rho_{\max}^3]\bigr)
    \]

\section{Full Proofs and Derivations}

\subsection{Proof of Theorem~\ref{thm:geometric_decomposition}}
\label{app:proof_thm1}
\begin{proof}
We utilize the decomposition of the total error relative to the Truncated THA model:

    \[
    \widehat{Y}^{\mathrm{THA}}-\widehat{Y}^{\mathrm{full}} = \big(\widehat{Y}^{\mathrm{THA}} - \widehat{Y}^{\mathrm{Trunc}}\big) + \big(\widehat{Y}^{\mathrm{Trunc}} - \widehat{Y}^{\mathrm{full}}\big) = \Delta_{\text{Bake}} + \Delta_{\text{Bias}}.
    \]
    
We analyze the squared normalized empirical Frobenius norm:

    \[
        \big\|\widehat{Y}^{\mathrm{THA}}- \widehat{Y}^{\mathrm{full}}\big\|_{F,N}^2 = \big\|\Delta_{\text{Bake}} + \Delta_{\text{Bias}}\big\|_{F,N}^2.
    \]
    
Expanding the squared norm using the associated inner product \(\langle \cdot, \cdot \rangle_{F,N}\):

    \begin{align*} \big\|\Delta_{\text{Bake}} + \Delta_{\text{Bias}}\big\|_{F,N}^2 &= \|\Delta_{\text{Bake}}\|_{F,N}^2 + \|\Delta_{\text{Bias}}\|_{F,N}^2 + 2 \big\langle \Delta_{\text{Bake}}, \Delta_{\text{Bias}} \big\rangle_{F,N}. \end{align*}
    
By the definition of Baking Alignment, \(C_{\text{Align}} = \langle \Delta_{\text{Bake}}, -\Delta_{\text{Bias}} \rangle_{F,N}\). By linearity of the inner product, the cross-term is

    \[
        2 \big\langle \Delta_{\text{Bake}}, \Delta_{\text{Bias}} \big\rangle_{F,N} = -2 C_{\text{Align}}.
    \]
    
Substituting this back into the expansion yields the theorem statement.
\end{proof}
\subsection{Proof of Theorem~\ref{thm:triangle_bound}}
\label{app:proof_thm2}
\begin{proof}
We decompose the difference between the THA prediction and the full MVMALS prediction. Since \(\sum_{k=1}^K \omega_k = 1\), we introduce the truncated full model terms \(U P_{S_k}B^\star\):

    \begin{align*}
        \widehat{Y}^{\mathrm{THA}}-\widehat{Y}^{\mathrm{full}} &= \sum_{k=1}^K \omega_k U\widehat{B}^{(k)} - U B^\star \\
        &= \sum_{k=1}^K \omega_k \left(U\widehat{B}^{(k)} - U P_{S_k}B^\star\right) + \left(\sum_{k=1}^K \omega_k U P_{S_k}B^\star - U B^\star\right).
    \end{align*}
    
The first term is the Empirical Baking Adjustment \(\Delta_{\text{Bake}}\). The second term is the Empirical Coverage Bias \(\Delta_{\text{Bias}} = U(C_\omega - I)B^\star = -UT\). We apply the triangle inequality:

    \[
        \big\|\widehat{Y}^{\mathrm{THA}}-\widehat{Y}^{\mathrm{full}}\big\|_{F,N} \le \big\|\Delta_{\text{Bake}}\big\|_{F,N} + \big\|\Delta_{\text{Bias}}\big\|_{F,N}.
    \]

\textit{Bounding the Estimation Gap:} By the convexity of the norm and Jensen's inequality (as established in Section 4.2.3),

    \[
        \big\|\Delta_{\text{Bake}}\big\|_{F,N} \le \left(\sum_{k=1}^K \omega_k (\varepsilon'_k)^2\right)^{1/2}.
    \]

\textit{Bounding the Bias Term:} We bound the norm \(\|UT\|_{F,N}\).
    \[
        \|UT\|_{F,N}^2 = \frac{1}{N} \|UT\|_F^2 = \frac{1}{N} \mathrm{Tr}(T^\top U^\top U T) = \mathrm{Tr}(T^\top G T).
    \]
    
Let \(T_j\) be the \(j\)-th column of \(T\).

    \[
        \mathrm{Tr}(T^\top G T) = \sum_{j=1}^l T_j^\top G T_j.
    \]
    
Using the Rayleigh quotient property for the positive semidefinite matrix \(G\), we have \(T_j^\top G T_j \le \lambda_{\max}(G) \|T_j\|_2^2\).

    \begin{align*}
        \mathrm{Tr}(T^\top G T) &\le \lambda_{\max}(G) \sum_{j=1}^l \|T_j\|_2^2 = \lambda_{\max}(G) \|T\|_F^2.
    \end{align*}

Taking the square root yields the bound on the bias term. Combining the bounds completes the proof.
\end{proof}

\subsection{Proof of Theorem~\ref{thm:improved_bound}}
\label{app:proof_thm3}
\begin{proof}
We start from the exact geometric decomposition established in Theorem~\ref{thm:geometric_decomposition}:

    \[
        \big\|\widehat{Y}^{\mathrm{THA}}-\widehat{Y}^{\mathrm{full}}\big\|_{F,N}^2 = \|\Delta_{\text{Bias}}\|_{F,N}^2 + \|\Delta_{\text{Bake}}\|_{F,N}^2 - 2\ C_{\text{Align}}.
    \]

We utilize the convex bound on the magnitude of the baking adjustment derived in Section 4.2.3:

    \[
        \|\Delta_{\text{Bake}}\|_{F,N}^2 \le \sum_{k=1}^K \omega_k (\varepsilon'_k)^2.
    \]

Substituting this inequality into the geometric decomposition yields the theorem statement.
\end{proof}

\end{document}